\documentclass{article}

\usepackage[preprint]{neurips_2025}

\usepackage[utf8]{inputenc} %
\usepackage[T1]{fontenc}    %
\usepackage{url}            %
\usepackage{booktabs}       %
\usepackage{amsfonts}       %
\usepackage{nicefrac}       %
\usepackage{microtype}      %
\usepackage{xcolor}         %
\usepackage{graphicx}
\usepackage{wrapfig}
\usepackage{comment}
\usepackage{tikz}
\usepackage{pgfplots}
\usepackage{wrapfig}

\usepgfplotslibrary{colorbrewer}
\usepgfplotslibrary{groupplots}
\usepgfplotslibrary{fillbetween}
\pgfplotsset{cycle list/Set1-8}
\pgfplotsset{compat=1.18}
\usetikzlibrary{circuits.logic.US, positioning, matrix}
\usepackage{circuitikz}
\usepackage{todonotes}
\usepackage{amsmath}
\usepackage{algorithm}      %
\usepackage{algpseudocode}  %

\usepackage{amsmath,amssymb} %
\usepackage{amsthm}
\newtheorem{theorem}{Theorem}

\title{Scalable Interconnect Learning in Boolean Networks}

\author{%
  Fabian Kresse \\
  Institute of Science and Technology Austria\\
  Klosterneuburg, AT 3400 \\
  \texttt{fabian.kresse@ista.ac.at} \\
  \And
  Emily Yu \\
  Institute of Science and Technology Austria \\
  Klosterneuburg, AT 3400 \\
  \texttt{emily.yu@ista.ac.at} \\
  \And
  Christoph H. Lampert \\
  Institute of Science and Technology Austria \\
  Klosterneuburg, AT 3400 \\
  \texttt{chl@ista.ac.at} \\
}

\begin{document}

\maketitle

\begin{abstract} 
Learned Differentiable Boolean Logic Networks (DBNs) already deliver efficient inference on resource-constrained hardware. We extend them with a trainable, differentiable interconnect whose parameter count remains constant as input width grows, allowing DBNs to scale to far wider layers than earlier learnable-interconnect designs while preserving their advantageous accuracy. To further reduce model size, we propose two complementary pruning stages: an SAT-based logic equivalence pass that removes redundant gates without affecting performance, and a similarity-based, data-driven pass that outperforms a magnitude-style greedy baseline and offers a superior compression–accuracy trade-off.
\end{abstract}

\section{Introduction}

Deploying deep neural networks on edge hardware is often limited by the silicon area and energy demanded by multiply‑accumulate (MAC) units. Architectures learned directly in the Boolean domain, where weights and activations are binary and computation reduces to simple logic gates, sidestep MACs altogether.

Petersen et al. \cite{petersen2022deep,petersen2024convolutional} first introduced Differentiable Logic Networks (DLNs), training them by relaxing each logic gate into a softmax mixture over the 16 possible Boolean functions and learning a weight for each. After training, a DLN operates entirely on hard, two-valued logic (0 or 1). Inputs are discretized, propagated through successive Boolean gates (e.g., XOR, OR), and the final class label is obtained by majority vote over grouped Boolean outputs. Because every operation is a simple logic gate, the resulting circuit can be mapped directly to an FPGA or ASIC, enabling ultra-fast, energy-efficient inference. For instance, when extended to convolution-style logic trees \cite{petersen2024convolutional}, this scheme reaches \(\approx 80\%\) CIFAR-10 accuracy in just 24 ns per image on an FPGA. We refer to any such hardened Boolean circuit, whether produced by a DLN or by later variants, as a Deep Boolean Network (DBN). The explicit circuit structure of DBNs also lends itself to SAT-based symbolic verification \cite{kresse2025}, making them well-suited to safety-critical embedded applications.

The original DLN architecture fixes the connections between layers at random, even though input features rarely carry equal information. On MNIST, for example, corner pixels add virtually no class signal. To overcome this, Bacellar et al. \cite{bacellar2024differentiable} proposed Deep Weightless Networks (DWNs), which use a straight-through differentiable training rule to learn both higher-arity Boolean gates and the input-to-first-layer interconnect, achieving Pareto gains in gate count and accuracy. The trade-off is that DWNs encode the interconnect as a dense parameter matrix whose size grows linearly with input width, limiting the scalability of layer width. Note that DLNs must be hardened after training to become DBNs, whereas DWNs remain in purely Boolean form throughout. Nevertheless, both ultimately yield DBNs and share the same deployment pipeline.

In this work, we build on these foundations by introducing a scalable learnable interconnect whose parameter count does not grow with the number of input bits. Our interconnect learning scheme recovers baseline accuracy and, more crucially, enables scaling of the learnable interconnect to wider layer sizes. As our method scales gracefully with input layer size, we also, for the first time, systematically investigate the impact of training the interconnect of layers beyond the first layer. Furthermore, in order to reduce network sizes, we investigate pruning schemes based on proofs of logical equivalence and data-driven heuristics that shrink the learned circuits. We show that for small network sizes, pruned networks achieve the same accuracy as directly training smaller networks of this size.

\pgfplotsset{
    discard if not/.style 2 args={
        x filter/.code={
            \edef\tempa{\thisrow{#1}}
            \edef\tempb{#2}
            \ifx\tempa\tempb
            \else
                \def\pgfmathresult{inf}
            \fi
        }
    }
}

\makeatletter
\pgfplotstableset{
    discard if not/.style 2 args={
        row predicate/.code={
            \def\pgfplotstable@loc@TMPd{\pgfplotstablegetelem{##1}{#1}\of}
            \expandafter\pgfplotstable@loc@TMPd\pgfplotstablename
            \edef\tempa{\pgfplotsretval}
            \edef\tempb{#2}
            \ifx\tempa\tempb
            \else
                \pgfplotstableuserowfalse
            \fi
        }
    }
}
\makeatother

\section{Differentiable Interconnect Learning}

In this section, we first give a formal definition of Differentiable Boolean Logic Networks (DBNs). We then present our differentiable interconnect learning framework, augmented with a gate-sampling mechanism that enables exploration of the whole parameter space, and we discuss how we extend interconnect learning beyond the first layer.

\subsection{Deep Differentiable Boolean Networks}

Let \(\mathcal{G}=(\mathcal{V},\mathcal{E})\) be a directed acyclic graph with vertices grouped in layers. 
Each vertex \(v\in\mathcal{V}\) (henceforth also gate) computes a Boolean function 
\(f_v:\{0,1\}^{k}\!\to\{0,1\}\) of fixed arity \(k\) (two for classic DLNs \cite{petersen2022deep}, possibly more in DWNs \cite{bacellar2024differentiable}). Edges \(\mathcal{E}\) route binary signals from one layer to the next; the set of incoming edges to \(v\) selects the \(k\) ordered input variables for \(f_v\). While each gate has \(k\)-input edges, it may have an arbitrary number of output edges. This framework fits both DLNs and DWNs.

\subsection{Interconnect Learning}

    \begin{wrapfigure}{r}{0.5\textwidth}
    
  \begin{center}
\begin{circuitikz}[
    american,
    scale=0.8, transform shape,   %
    every node/.style={font=\footnotesize},
    >=stealth, line cap=round
]

\node[and port,
      rotate=-90,
      scale=0.6,
      label={[rotate=0,                 %
              yshift=3.7mm,             %
              font=\large]              %
             center:$1$}
     ] (u1) at (0,0) {};
     
\node[and port, rotate=-90, scale=0.6,
      label={[rotate=0,                 %
              yshift=3.7mm,             %
              font=\large]              %
             center:$2$}
     ] (u2) at (0.8,0) {};
\node[and port, rotate=-90, scale=0.6,
      label={[rotate=0,                 %
              yshift=3.7mm,             %
              font=\large]              %
             center:$3$}
     ] (u3) at (1.6,0) {};
\node[and port, rotate=-90, scale=0.6,
      label={[rotate=0,                 %
              yshift=3.7mm,             %
              font=\large]              %
             center:$4$}
     ] (u4) at (2.4,0) {};
\node[and port, rotate=-90, scale=0.6,
      label={[rotate=0,                 %
              yshift=3.7mm,             %
              font=\large]              %
             center:$5$}
     ] (u5) at (3.2,0) {};
\node[and port, rotate=-90, scale=0.6,
      label={[rotate=0,                 %
              yshift=3.7mm,             %
              font=\large]              %
             center:$6$}
     ] (u6) at (4.0,0) {};

\node[and port, scale=0.9,rotate=-90,
      label={[rotate=0,                 %
              yshift=5.7mm,             %
              font=\large]              %
             center:$v$}] (U) at (2.3,-2.0) {};

\draw[line width=0.69pt] (u1.out) -- (U.in 1);   %
\draw[line width=1.34pt] (u6.out) -- (U.in 1);   %
\draw[line width=1.5pt] (u3.out) -- (U.in 1);   %
\draw[line width=1.2pt] (u4.out) -- (U.in 1);   %

\draw[line width=0.43pt] (u1.out) -- (U.in 2);   %
\draw[line width=0.5pt] (u6.out) -- (U.in 2);   %
\draw[line width=1.6pt] (u3.out) -- (U.in 2);   %
\draw[line width=1.3pt] (u4.out) -- (U.in 2);   %

\node[anchor=west] at (4.6,  0.4) {$\mathbf{candidates} : [1,3,4,6]$};
\node[anchor=west] at (4.6,  -0.4) {$\boldsymbol{weights_1}\!: [0.1,0.8,0.5,0.2]$};
\node[anchor=west] at (4.6, -1.2) {$\boldsymbol{weights_2}\!: [0.3,0.7,0.4,0.6]$};

\end{circuitikz}

  \end{center}
  \caption{Illustration of the learnable interconnect mechanism. Candidates denote the tensor storing the possible connections of the gate $v$ to the previous layer. Weights store the connection strength of inputs $1$ and $2$ to the specific gates in the candidates tensor.}
  \label{fig:learnable interconnect}
\end{wrapfigure}
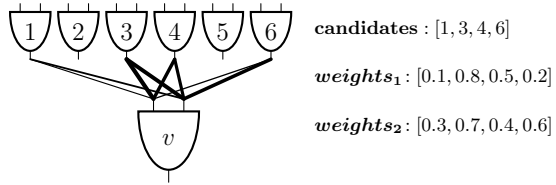

Each gate must choose its \(k\) inputs from the \(I\) binary signals produced by the previous layer (or the binarized input). A naïve, fully parameterised approach assigns a real-valued weight to every possible connection, requiring a learnable weight tensor \( W \in \mathbb{R}^{G \times k \times I}\) whose axes index gates \(g = 1,\dots,G\), input slots \(j = 1,\dots,k\), and candidate bits \(i = 1,\dots,I\).
Recent works utilize this full-sized parameter tensor $W$, or a subset, to assign connections. Bacellar et al.~\cite{bacellar2024differentiable} and Yue et al. \cite{yue2024learninginterpretabledifferentiablelogic}, adopt a straight-through estimator (STE): the forward pass takes the hard gate input-wise \(\arg\max\) decision, while the backward pass is computed according to a Softmax.

The batch-wise straight-through estimator (STE) gradient for the connection weights used by \cite{bacellar2024differentiable}, and adopted in our work, is defined as follows. With a mini-batch of binary inputs \(x \in \{0,1\}^{B \times I}\) and upstream gradients \(dy \in \mathbb{R}^{B \times G \times k}\), the gradient of the loss with respect to the weight \(w_{g,j,i}\) for the connection \(i \!\rightarrow\! (g,j)\) is
\begin{equation}
\label{eq:gradient}
  g_{g,j,i} \;=\; \sum_{b=1}^{B}(2x_{b,i}-1)\,dy_{b,g,j}
\end{equation}
so more negative \(g_{g,j,i}\) suggests a greater potential loss reduction if that connection’s weight were increased.

The approach of Bacellar et al. \cite{bacellar2024differentiable} requires the full \(G \times k \times I\) weight tensor and thus scales quadratically when \(I\) is of similar magnitude as \(G\). In the proposed method, we reduce the interconnect to \(O(G)\) parameters by assigning each gate only \(C\!\ll\! I\) candidate inputs.  This approach of selecting only a subset of candidate indices is equivalent to the one taken in \cite{yue2024learninginterpretabledifferentiablelogic}. The weight tensor therefore has shape \(G \;\times\; k \;\times\; C \), reducing memory and computation by a factor of \(I/C\). We store the \(C\) candidate connections for each gate in an additional tensor. In line with~\cite{bacellar2024differentiable}, we employ the same STE gradient estimator, utilizing the hard \(\arg\max\) connection choice in the forward pass and updating connections with the surrogate gradient from Eq. \eqref{eq:gradient}. Figure \ref{fig:learnable interconnect} visualizes the learnable interconnect mechanism for $k=2$ and $C=4$.

\subsection{Gate Sampling}

As each gate observes only a small candidate set in our reduced interconnect, we periodically refresh that set to explore the full input space. Every \(\beta\) training steps we locate, for each gate, the \(R\) lowest-weighted connections—those gradient descent currently regards as least useful—and replace them. The freshly inserted candidates inherit the smallest weight that was \emph{not} replaced, so they can compete immediately. Pseudocode for this algorithm is given in Alg. \ref{alg:gate_sampling}.

For the replacement we propose two sampling functions. In the \emph{random} scheme the new connections are drawn uniformly from the entire set of \(I\) candidate inputs. In the \emph{gradient-guided} scheme we compute the surrogate gradients of Eq.~\eqref{eq:gradient} with respect to all \(I\) inputs for one mini-batch and insert the \(R\) candidates with the most negative gradients, i.e.\ those expected to reduce the loss most strongly. Although evaluating the full gradient costs \(O(I\,Gk)\) operations, we never need to materialize the complete weight tensor: we can maintain a running top-\(R\) set, keeping memory usage independent of \(I\). 

\begin{algorithm}[t]
\caption{Layer-wise gate-sampling refresh executed every $\beta$ training steps. For each gate-slot pair, the algorithm replaces the $R$ weakest candidate connections with newcomers drawn by either a \emph{random} or \emph{gradient-guided} sampling rule.}
\label{alg:gate_sampling}
\begin{algorithmic}[1]
\Require
  weight tensor \(W \in \mathbb{R}^{G\times k\times C}\)           \Comment{trainable connections}
\Require
  candidate tensor \(\mathcal{C} \in \{1,\dots,I\}^{G\times k\times C}\) \Comment{global input indices}
\Require
  current step, refresh interval \(\beta\), replacement budget \(R\), last–batch \((x,dy)\)
  \If{\( \text{step} \bmod \beta = 0 \)} \Comment{refresh every \(\beta\) steps}
    \For{$g=1$ \textbf{to} $G$}
      \For{$j=1$ \textbf{to} $k$}
        \State $L\gets R$ smallest indices in $W[g,j,:]$
        \State $w_{\text{floor}} \gets \min_{c\notin L} W[g,j,c]$
        \State $N \gets \Call{Sample}{R,I,g,j,x,dy}$ \Comment{$R$ new global indices}
        \For{$r=1$ \textbf{to} $R$}         \Comment{replace each weak candidate}
            \State $\mathcal{C}[g,j,L_r]\gets N_r$
            \State $W[g,j,L_r]\gets w_\text{floor}$
        \EndFor
      \EndFor
    \EndFor
  \EndIf
\end{algorithmic}
\end{algorithm}
\subsection{Layer-Wise Optimization}

We propose a full layer-wise schedule for optimizing the interconnect of all layers: training starts with the first layer’s connection weights unfrozen, while all gates in the entire network are trainable. After a fixed number of iterations, both the connections and the gates of that layer are frozen, and the procedure moves on to the next layer, unfreezing the corresponding interconnect. The process repeats until every layer has optimised its interconnect. A final pass then fine-tunes all gate parameters while keeping every interconnect fixed. This sequential strategy allows each layer to benefit from the stable representations established below it. %

\section{Pruning}

We present four pruning techniques tailored to DBNs. The first, \emph{trivial pruning}, deletes gates whose outputs are never consumed by any downstream logic. The remaining three methods are novel: \emph{logic equivalence pruning} detects semantically redundant gates using SAT-based analysis, hence subsuming the constant-gate pruning used in earlier work \cite{bacellar2024differentiable,petersen2024convolutional}. The other two methods, \emph{greedy} and \emph{similarity}-based pruning, utilize data-driven criteria to reduce gate count. As traditional pruning techniques typically rely on some notion of connection magnitude, which is not directly applicable in DBNs, we utilize the magnitude-inspired greedy pruning scheme as our data-driven pruning baseline.

\paragraph{Logic equivalence pruning.}
To eliminate redundant computation, we remove gates that realise the same Boolean function within a layer. Each input bit \(x_i\) is treated as a propositional variable and propagated symbolically through the network, so that every gate \(f_g\) carries an explicit Boolean expression \(f_g(x_1,\dots,x_I)\). Two gates are \emph{logically equivalent}, written \(f_i \equiv f_j\), if their expressions evaluate identically on \emph{all} possible input assignments, that is, \(f_i(\mathbf{x}) = f_j(\mathbf{x})\) for all \(\mathbf{x} \in \{0,1\}^I\).  If two such gates exist in a layer, we may safely delete one and reroute its fan-out to the other without changing the network's output. As there are no logically equivalent gates left after this, we only have to perform this sweep once. The soundness of the procedure is formalised in Theorem~\ref{thm:soundness}.

\vspace{2pt}
\begin{theorem}[Soundness of logic equivalence pruning]
\label{thm:soundness}
Let a DBN layer \(L_\ell=(f_1,\dots,f_G)\) contain \(f_i\equiv f_j\).
Delete \(f_j\) and redirect every edge that reads \(f_j\) to read \(f_i\),
yielding a network \(\mathcal N'\).
Then \(\mathcal N(\mathbf x)=\mathcal N'(\mathbf x)\;\forall\mathbf x\in\{0,1\}^I\).
\end{theorem}

\begin{proof}
Because \(f_i(\mathbf x)=f_j(\mathbf x)\) for all inputs, replacing each use of \(f_j\) by \(f_i\) leaves the truth value on every downstream connection unchanged; induction over layers completes the proof. \qedhere
\end{proof}

As a naïve all-pairs SAT check for determining logical equivalence costs \(O(G^{2})\) SAT-checks, our implementation applies a two-stage process: (i) \texttt{z3.simplify} performs heuristic, but sound, simplification of each $f$ followed by a hash (sum of surviving variable indices) to bucket gates; (ii) full SAT queries are only performed inside each bucket. Since the simplification procedure is sound, logic equivalence pruning remains sound as well.

Using SAT queries to merge semantically identical gates is established practice in logic synthesis (e.g.\ SAT sweeping~\cite{kuehlmann2004}). However, DBN layers differ from typical netlists: they are extremely \emph{wide} yet only a few levels deep.  In this regime the proposed simplify-plus-hash filter eliminates the vast majority of candidate pairs before any SAT call, giving the same pruning decisions as an exhaustive check in our experiments.

\paragraph{Greedy Pruning} Inspired by magnitude-based pruning techniques in traditional neural networks~\cite{han2015learning}, we perform data-driven pruning. During a forward pass over the training set, we track the $\mathtt{0}$ and $\mathtt{1}$ activations for each gate. If one of the two values occurs in at least a user-specified fraction of the cases (e.g., \(95\%\)), the gate is considered effectively constant and is replaced by that constant. This frequency threshold lets us control how “close” to a constant the output must be before pruning. This simple rule is the closest discrete analogue to magnitude pruning and therefore serves as our baseline for the subsequent data-driven pruning scheme.

\paragraph{Similarity Pruning} Additionally, we propose a pruning technique based on functional redundancy between gates. First, we perform a forward pass over the entire training set and record each gate’s binary activation ($\mathtt{0}$ or $\mathtt{1}$) for every input sample, yielding an activation vector for each gate. Next, for every pair of gates $(i,j)$, we compute the Spearman rank correlation coefficient $\rho_{ij}$ between their activation vectors. Given a correlation threshold $c$, we identify all pairs with $\rho_{ij} \ge c$ as highly similar. For each such pair, we remove one gate by redirecting the next-layers fan-ins that rely on this gate to its correlated counterpart fan-outs, thereby hopefully reducing the overall gate count without significantly altering the network’s behavior.

\section{Experiments} 
\label{sec:experiments}
\pgfplotsset{
    discard if not/.style 2 args={
        x filter/.code={
            \edef\tempa{\thisrow{#1}}
            \edef\tempb{#2}
            \ifx\tempa\tempb
            \else
                \def\pgfmathresult{inf}
            \fi
        }
    }
}

\makeatletter
\pgfplotstableset{
    discard if not/.style 2 args={
        row predicate/.code={
            \def\pgfplotstable@loc@TMPd{\pgfplotstablegetelem{##1}{#1}\of}
            \expandafter\pgfplotstable@loc@TMPd\pgfplotstablename
            \edef\tempa{\pgfplotsretval}
            \edef\tempb{#2}
            \ifx\tempa\tempb
            \else
                \pgfplotstableuserowfalse
            \fi
        }
    }
}
\makeatother

All experiments were performed on CIFAR-10 using the Adam optimizer \cite{kingma2014adam} with an initial learning rate of $10^{-2}$, decayed to $10^{-5}$ via a cosine schedule, and a batch size of 100. For binarization, we adopt Distributive Thermometer Encoding with 10 thresholds \cite{bacellar2024differentiable}. Unless stated otherwise, our network comprises three DWN layers ($k = 2$), each with 12\,000 units; we set the GroupSum temperature to $\tau = 30$ and the interconnect parameters to $C = 8$, $R = 4$, and $\beta = 20$. Reported results are averaged over three random network initialization seeds with fixed train, validation and test splits (error bars denote one standard deviation). All runtimes were measured on an NVIDIA A10 GPU. Our implementation is written in PyTorch (2.1) and extends the public codebase of \cite{bacellar2024differentiable}.

\subsection{Scalable Learnable Interconnect}
\begin{figure}[t!]
    \centering

\begin{tikzpicture}
\begin{axis}[
    width=7.0cm,
    height=4cm,
    xlabel={Wall Clock Time (s)},
    ylabel={Test Accuracy},
    ymin=0.45, ymax=0.58,
    xmin=0,   xmax=2000,        %
    title={Test Accuracy vs.\ Wall Clock Time},
    grid=both,
    legend style={
        at={(1.02,0.5)},        %
      anchor=west,
        legend columns=1,
        cells={align=center},
    },
    every axis plot/.append style={very thick},
    %every axis legend/.append style={font=\footnotesize},
]
  \addplot+[
    discard if not={run}{learnable_8_none_20_12000_0},
    mark=none,
    name path=upper,
    opacity=0.3,
    forget plot
  ]
  table[
    col sep=comma,
    trim cells=true,
    x=time,
    y expr=\thisrow{test_mean} + \thisrow{test_std},
  ]{figures/figure_1.csv};

  \addplot+[
    discard if not={run}{learnable_8_none_20_12000_0},
    mark=none,
    name path=lower,
    opacity=0.3,
    forget plot
  ]
  table[
    col sep=comma,
    trim cells=true,
    x=time,
    y expr=\thisrow{test_mean} - \thisrow{test_std},
  ]{figures/figure_1.csv};

  \addplot [ %
    fill opacity=0.2,
    draw=none,
    forget plot
  ] fill between[
    of=upper and lower, %
  ];

  \addplot+[
    discard if not={run}{learnable_8_none_20_12000_0},
    mark=none,
    solid,
  ]
  table[
    col sep=comma,
    trim cells=true,
    x=time,
    y=test_mean,
  ]{figures/figure_1.csv};
  \addlegendentry{Full-Interconnect Learnable (\cite{bacellar2024differentiable})}

  \addplot+[
    discard if not={run}{random_8_none_20_12000_0},
    mark=none,
    name path=upper,
    opacity=0.3,
    forget plot
  ]
  table[
    col sep=comma,
    trim cells=true,
    x=time,
    y expr=\thisrow{test_mean} + \thisrow{test_std},
  ]{figures/figure_1.csv};

  \addplot+[
    discard if not={run}{random_8_none_20_12000_0},
    mark=none,
    name path=lower,
    opacity=0.3,
    forget plot
  ]
  table[
    col sep=comma,
    trim cells=true,
    x=time,
    y expr=\thisrow{test_mean} - \thisrow{test_std},
  ]{figures/figure_1.csv};

  \addplot [
    fill opacity=0.2,
    draw=none,
    forget plot
  ] fill between[
    of=upper and lower, %
  ];

  \addplot+[
    discard if not={run}{random_8_none_20_12000_0},
    mark=none,
    solid,
  ]
  table[
    col sep=comma,
    trim cells=true,
    x=time,
    y=test_mean,
  ]{figures/figure_1.csv};
  \addlegendentry{Fixed Interconnect (\cite{petersen2022deep})}

  \addplot+[
    discard if not={run}{learnable_ste_8_none_20_12000_0},
    mark=none,
    name path=upper,
    opacity=0.3,
    forget plot
  ]
  table[
    col sep=comma,
    trim cells=true,
    x=time,
    y expr=\thisrow{test_mean} + \thisrow{test_std},
  ]{figures/figure_1.csv};

  \addplot+[
    discard if not={run}{learnable_ste_8_none_20_12000_0},
    mark=none,
    name path=lower,
    opacity=0.3,
    forget plot
  ]
  table[
    col sep=comma,
    trim cells=true,
    x=time,
    y expr=\thisrow{test_mean} - \thisrow{test_std},
  ]{figures/figure_1.csv};

  \addplot [
    fill opacity=0.2,
    draw=none,
    forget plot
  ] fill between[
    of=upper and lower, %
  ];

  \addplot+[
    discard if not={run}{learnable_ste_8_none_20_12000_0},
    mark=none,
    solid,
  ]
  table[
    col sep=comma,
    trim cells=true,
    x=time,
    y=test_mean,
  ]{figures/figure_1.csv};
  \addlegendentry{Subset-Connections Learnable \cite{yue2024learninginterpretabledifferentiablelogic})}

  \addplot+[
    discard if not={run}{learnable_ste_8_gradient_estimate_20_12000_0},
    mark=none,
    name path=upper,
    opacity=0.3,
    forget plot
  ]
  table[
    col sep=comma,
    trim cells=true,
    x=time,
    y expr=\thisrow{test_mean} + \thisrow{test_std},
  ]{figures/figure_1.csv};

  \addplot+[
    discard if not={run}{learnable_ste_8_gradient_estimate_20_12000_0},
    mark=none,
    name path=lower,
    opacity=0.3,
    forget plot
  ]
  table[
    col sep=comma,
    trim cells=true,
    x=time,
    y expr=\thisrow{test_mean} - \thisrow{test_std},
  ]{figures/figure_1.csv};

  \addplot [ %
    fill opacity=0.2,
    draw=none,
    forget plot
  ] fill between[
    of=upper and lower, %
  ];

  \addplot+[
    discard if not={run}{learnable_ste_8_gradient_estimate_20_12000_0},
    mark=none,
    solid,
  ]
  table[
    col sep=comma,
    trim cells=true,
    x=time,
    y=test_mean,
  ]{figures/figure_1.csv};
  \addlegendentry{Gradient-Guided Sampling (proposed)}

  \addplot+[
    discard if not={run}{learnable_ste_8_uniform_20_12000_0},
    mark=none,
    name path=upper,
    opacity=0.3,
    forget plot
  ]
  table[
    col sep=comma,
    trim cells=true,
    x=time,
    y expr=\thisrow{test_mean} + \thisrow{test_std},
  ]{figures/figure_1.csv};

  \addplot+[
    discard if not={run}{learnable_ste_8_uniform_20_12000_0},
    mark=none,
    name path=lower,
    opacity=0.3,
    forget plot
  ]
  table[
    col sep=comma,
    trim cells=true,
    x=time,
    y expr=\thisrow{test_mean} - \thisrow{test_std},
  ]{figures/figure_1.csv};

  \addplot [
    fill opacity=0.2,
    draw=none,
    forget plot
  ] fill between[
    of=upper and lower, %
  ];

  \addplot+[
    discard if not={run}{learnable_ste_8_uniform_20_12000_0},
    mark=none,
    solid,
  ]
  table[
    col sep=comma,
    trim cells=true,
    x=time,
    y=test_mean,
  ]{figures/figure_1.csv};
  \addlegendentry{Random Sampling (proposed)}

\end{axis}

\end{tikzpicture}

    \caption{Test set accuracy on CIFAR-10 for the first 2000 seconds of wall clock time for the proposed method with gradient-guided and random sampling, the method from \cite{bacellar2024differentiable} with the full interconnect, the interconnect restricted to \(C=8\) candidates \cite{yue2024learninginterpretabledifferentiablelogic}, and the original fixed interconnect from \cite{petersen2022deep}, all trained in the DWN framework \cite{bacellar2024differentiable}. Only the first layer's interconnect is learned.}
    \label{fig:interconnect-firstlayer}
\end{figure}
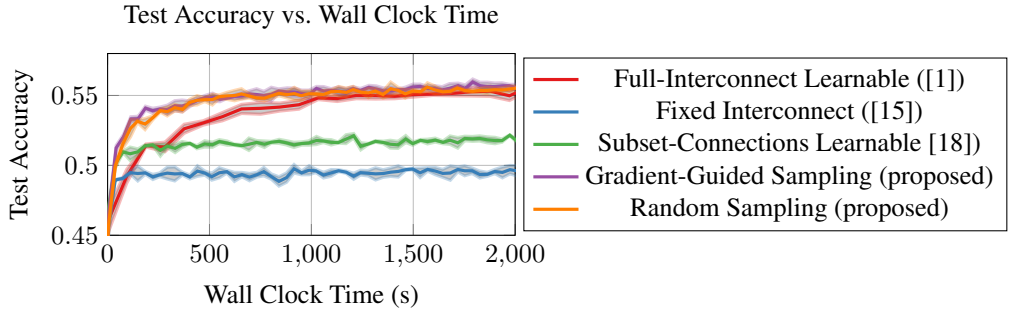

Figure~\ref{fig:interconnect-firstlayer} compares five interconnect schemes when \emph{only the first layer’s} interconnect is trainable: the static fixed random interconnect baseline; the full \(O(GI)\) learnable interconnect of Bacellar et al.~\cite{bacellar2024differentiable}, which optimizes all possible connections via the surrogate gradient; and our scalable \(O(G)\) interconnect, evaluated under random sampling and gradient-guided sampling. Furthermore, the Figure shows training with 8 candidate connections and no resampling as proposed by \cite{yue2024learninginterpretabledifferentiablelogic}. Each method was allotted a two‐hour wall‐clock training budget on an NVIDIA A10 GPU and ran until that budget was exhausted. On a per‐epoch basis, the full learnable interconnect from~\cite{bacellar2024differentiable} achieves the fastest accuracy gains. This is shown in the Appendix, Figure \ref{fig:convergence in epochs}. However, when measured in wall‐clock time, our $O(G)$ interconnect converges more quickly, compensating for the slower per-epoch gains.

\pgfplotsset{
    discard if not/.style 2 args={
        x filter/.code={
            \edef\tempa{\thisrow{#1}}
            \edef\tempb{#2}
            \ifx\tempa\tempb
            \else
                \def\pgfmathresult{inf}
            \fi
        }
    }
}

\makeatletter
\pgfplotstableset{
    discard if not/.style 2 args={
        row predicate/.code={
            \def\pgfplotstable@loc@TMPd{\pgfplotstablegetelem{##1}{#1}\of}
            \expandafter\pgfplotstable@loc@TMPd\pgfplotstablename
            \edef\tempa{\pgfplotsretval}
            \edef\tempb{#2}
            \ifx\tempa\tempb
            \else
                \pgfplotstableuserowfalse
            \fi
        }
    }
}
\makeatother
\begin{wrapfigure}[22]{r}{0.5\textwidth}
  \centering

\begin{tikzpicture}

\begin{groupplot}[                                     %
    group style={
        group size=1 by 1,              %
        vertical sep=2.2cm,
        group name=myplotshello,             %
    },
    every axis plot/.append style={very thick},
    width=7cm,
    height=5cm,
    grid=both,
    legend to name=sharedlegend,        %
    legend style={
    legend columns=3,
    cells={align=center},
    },
]

\nextgroupplot[
    xlabel={Gates per Layer},
    ylabel={Accuracy},
    ymin=0.42, ymax=0.57,
    title={Test Acc.},
    xmin=1500, xmax=12000,
    scaled x ticks=false,
    xtick={3000,6000,12000},
]

\addplot+[
    discard if not={key}{1},
    mark=none,
    error bars/.cd,
        y dir=both,
        y explicit,
]
table[
    col sep=comma,
    trim cells=true,
    x=size,
    y=test,
    y error=test_std,
]{figures/figure_22506.csv};
\addlegendentry{$L=1$}
\addplot+[
    discard if not={key}{2},
    mark=none,
    error bars/.cd,
        y dir=both,
        y explicit,
]
table[
    col sep=comma,
    trim cells=true,
    x=size,
    y=test,
    y error=test_std,
]{figures/figure_22506.csv};
\addlegendentry{$L=2$}
\addplot+[
    discard if not={key}{3},
    mark=none,
    error bars/.cd,
        y dir=both,
        y explicit,
]
table[
    col sep=comma,
    trim cells=true,
    x=size,
    y=test,
    y error=test_std,
]{figures/figure_22506.csv};
\addlegendentry{$L=3$}

\end{groupplot}
\path (myplotshello c1r1.south) ++(-0.0cm,-1.0cm)     %
      node[anchor=north] {\pgfplotslegendfromname{sharedlegend}};

\end{tikzpicture}
\caption{Test accuracy for interconnects up to layer $L=\{1,2,3\}$ trainable. Gates per layer on the x-axis: $\{1500,3000,6000,12000\}$.} 
\label{fig:overfitting}
\end{wrapfigure}
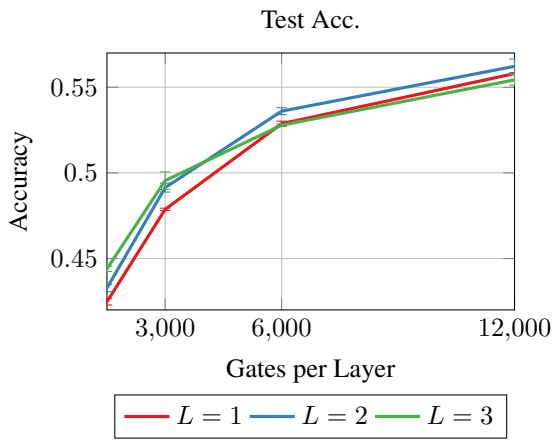

Table \ref{tab:diff_sizes} shows the final test set accuracies for the same networks as investigated in Figure \ref{fig:interconnect-firstlayer}, including results for larger network sizes. Our method recovers the same test-set accuracy as the full learnable interconnect from \cite{bacellar2024differentiable}. This is the case for both random sampling and the more involved gradient-based estimation sampling. For larger network sizes the gradient-based estimation sampling method improves upon the random sampling.

\paragraph{Scalability.}

Our $O(G)$ interconnect reduces the memory footprint from \cite{bacellar2024differentiable} full $O(GI)$ weight tensor (which requires on the order of gigabytes) by three orders of magnitude (only megabytes). Concretely, the memory cost in bytes is given for \cite{bacellar2024differentiable} as $
    M_{\mathrm{full}}
    = k \times G \times I \times 4$ versus ours: $M_{\mathrm{sparse}}
    = k \times G \times C \times 4 \times 2.$ Note the additional $\times2$ accounting for the needed candidate indices. With $k=2$, $I=30\,720$ input bits for CIFAR-10, and $C=8$ candidates.  At $G=12\,000$ gates this gives $\approx 3$ GB versus $ \approx 1.5$ MB in required memory. Doubling to $G=24\,000$ pushes $M_{\mathrm{full}}\approx5.9$ GB—beyond what remains on a 24GB A10 GPU when accounting for other tensors. In contrast, our $O(G)$ interconnect only grows to $\approx3$MB at 24000 gates, easily fitting and achieving the same accuracy for the smaller network (see Table~\ref{tab:diff_sizes}).

\begin{table}[t!]
\centering
\caption{Scalability with layer width on CIFAR-10. All runs are trained within a fixed time budget of 2 hours; the number of epochs that could be performed in this time is given in the column epochs. All networks consist of three layers of width $G$. Averaged over three different network initialization seeds.}
\label{tab:diff_sizes}
\begin{tabular}{lccc}
\toprule
Method &  Width $G$ & Acc. & Epochs  \\
\midrule
Fixed Interconnect \cite{petersen2022deep}          &  12\,000 & 0.498±0.004 & 1944\\
Fixed Interconnect \cite{petersen2022deep}          &  24\,000 & 0.513±0.004 & 1228 \\
\midrule
Full-Interconnect \cite{bacellar2024differentiable} &  12\,000 & 0.559±0.001 & 77 \\
\midrule
Subset-Connections Learnable \cite{yue2024learninginterpretabledifferentiablelogic}    &  12\,000 & 0.516±0.002 & 1954 \\
Subset-Connections Learnable \cite{yue2024learninginterpretabledifferentiablelogic}  &  24\,000 & 0.532±0.001 & 1178 \\
\midrule
Random Sampling (proposed) &  12\,000 & 0.559±0.001 & 2027 \\
Random Sampling (proposed) &  24\,000 & 0.566±0.002 & 1151 \\
\midrule
Gradient-Guided (proposed) & 12\,000 & 0.561±0.003 & 1203 \\
Gradient-Guided (proposed) & 24\,000 &\textbf{0.575±0.002} & 641 \\

\bottomrule
\end{tabular}

\medskip
\end{table}

\paragraph{Layer-Wise Optimization.}  

\begin{figure}
    \centering

\begin{tikzpicture}

\begin{groupplot}[                                     %
    group style={
        group size=2 by 1,              %
        vertical sep=2.2cm,
        group name=myplots,             %
    },
    every axis plot/.append style={very thick},
    width=7cm,
    height=4cm,
    grid=both,
    %legend to name=sharedlegendthree,        %
    %every axis legend/.append style={font=\footnotesize},
    legend style={
    legend columns=3,
    cells={align=center},
    },
]

\nextgroupplot[
    xlabel={Epoch},
    ylabel={Accuracy},
    xmin=0, xmax=2100,
    ymin=0.52, ymax=0.58,
    title={Validation Acc.},
    xtick={0, 666, 1000, 1332, 2100},
]

\addplot+[
    discard if not={run}{learnable_ste_8_gradient_estimate_20_12000_1},
    mark=none,
    solid,
  ]
  table[
    col sep=comma,
    trim cells=true,
    x=epoch,
    y=test_mean,
  ]{figures/overfitting2506.csv};
  
    \addplot+[
    discard if not={run}{learnable_ste_8_gradient_estimate_20_12000_2},
    mark=none,
    solid,
  ]
  table[
    col sep=comma,
    trim cells=true,
    x=epoch,
    y=test_mean,
  ]{figures/overfitting2506.csv};

    \addplot+[
    discard if not={run}{learnable_ste_8_gradient_estimate_20_12000_3},
    mark=none,
    solid,
  ]
  table[
    col sep=comma,
    trim cells=true,
    x=epoch,
    y=test_mean,
  ]{figures/overfitting2506.csv};

\nextgroupplot[
    xlabel={Epoch},
    ylabel={},
    xmin=0, xmax=2100,
    ymin=0.6, ymax=0.96,
    title={Train Acc.},
    xtick={0, 666,1000, 1332, 2100},
    legend to name=sharedlegendthree,
]

\addplot+[
    discard if not={run}{learnable_ste_8_gradient_estimate_20_12000_1},
    mark=none,
    solid,
  ]
  table[
    col sep=comma,
    trim cells=true,
    x=epoch,
    y=train_mean,
  ]{figures/overfitting2506.csv};
  \addlegendentry{$L=1$}
    \addplot+[
    discard if not={run}{learnable_ste_8_gradient_estimate_20_12000_2},
    mark=none,
    solid,
  ]
  table[
    col sep=comma,
    trim cells=true,
    x=epoch,
    y=train_mean,
  ]{figures/overfitting2506.csv};
\addlegendentry{$L=2$}
    \addplot+[
    discard if not={run}{learnable_ste_8_gradient_estimate_20_12000_3},
    mark=none,
    solid,
  ]
  table[
    col sep=comma,
    trim cells=true,
    x=epoch,
    y=train_mean,
  ]{figures/overfitting2506.csv};
\addlegendentry{$L=3$}
\end{groupplot}

\path (myplots c1r1.south) ++(3.22cm,-0.7cm)     %
      node[anchor=north] {\pgfplotslegendfromname{sharedlegendthree}};
\end{tikzpicture}
\caption{%
Accuracy under the three layer-wise interconnect–training schedules \(L\in\{1,2,3\}\) for the model with $12000$ gates per layer.  Averaged over 3 seeds. Standard deviation not shown. \textbf{Left:} validation accuracy. \textbf{Right:} training accuracy.} 
\label{fig:overfitting12000}
\end{figure}
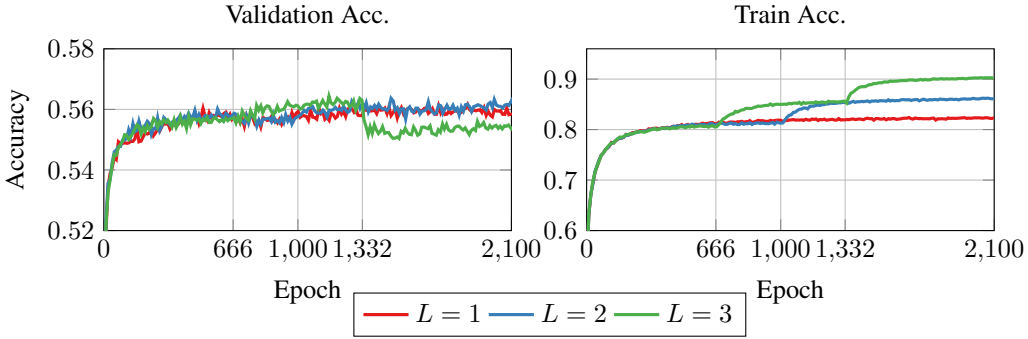

Due to the memory savings of our $O(G)$ interconnect, we can for the first time, train the interconnect beyond the first layer in very wide DBN models. We allocate a fixed budget of $E=2{,}100$ total epochs across the first $L$ layers by assigning $E_{\mathrm{train}} = \frac{2{,}000}{L}$ epochs to train each of the $L$ layers sequentially and $E_{\mathrm{FT}} = 100$ epochs to fine-tune all gates together. At the start of each of the epochs allocated to training one of the interconnects, we reinitialize our learning rate and restart cosine decay. Figure~\ref{fig:overfitting} plots the test accuracy after this schedule for $L\in\{1,2,3\}$ on networks with up to 12000 gates per layer. Train and validation accuracy for the largest model are shown in Figure~\ref{fig:overfitting12000}. 

We observe clear signs of overfitting beyond the second layer: on smaller networks (< 6000 gates per layer), training the interconnect of more layers yields modest gains, whereas for larger widths ($\geq6\,000$ gates), pushing beyond $L=2$ causes test accuracy to decline. Figure~\ref{fig:overfitting12000} shows that although the train‐set accuracy continues to rise with $L$, the validation‐set accuracy drops—clear signs of overfitting. Note that there are clear jumps in train and validation accuracy observable, when interconnect optimization of the next layer is enabled. For $L=2$ this jump occurs at 1000 epochs, for $L=3$ at $666$ and $1332$ epochs.

\subsection{Pruning}

\begin{figure}
    \centering
    \begin{tikzpicture}

\begin{groupplot}[                                     %
    group style={
        group size=1 by 3,              %
        vertical sep=0.7cm,
        horizontal sep=0.7cm,
        group name=myplotstwo,             %
    },
    every axis plot/.append style={very thick},
    width=14cm,
    height=4cm,
    grid=both,
    %legend to name=sharedlegendtwo,        %
    legend style={
    legend columns=4,
    cells={align=center},
    },
]

\nextgroupplot[
    xlabel={},
    ylabel={Accuracy},
    ymin=0.3, ymax=0.6,
    title={$L=1$},
    xmin=0, xmax=36000,
    scaled x ticks=false,
    xticklabels={},
]

\addplot+[
    discard if not={key}{Greedy_0},
    mark=o,
    color=orange
]
table[
    col sep=comma,
    trim cells=true,
    x=size,
    y=accuracy,
]{figures/plot_layers_1.csv};
\addplot+[
    discard if not={key}{Greedy_1},
    mark=o,
    color=orange
]
table[
    col sep=comma,
    trim cells=true,
    x=size,
    y=accuracy,
]{figures/plot_layers_1.csv};
\addplot+[
    discard if not={key}{Greedy_2},
    mark=o,
    color=orange
]
table[
    col sep=comma,
    trim cells=true,
    x=size,
    y=accuracy,
]{figures/plot_layers_1.csv};

\addplot+[
    discard if not={key}{Greedy_3},
    mark=o,
    color=orange
]
table[
    col sep=comma,
    trim cells=true,
    x=size,
    y=accuracy,
]{figures/plot_layers_1.csv};

\addplot+[
    discard if not={key}{Similarity_0},
    mark=o,
    color=blue,
]
table[
    col sep=comma,
    trim cells=true,
    x=size,
    y=accuracy,
]{figures/plot_layers_1.csv};
\addplot+[
    discard if not={key}{Similarity_1},
    mark=o,
    color=blue,            %
]
table[
    col sep=comma,
    trim cells=true,
    x=size,
    y=accuracy,
]{figures/plot_layers_1.csv};
\addplot+[
    discard if not={key}{Similarity_2},
    mark=o,
    color=blue,            %
]
table[
    col sep=comma,
    trim cells=true,
    x=size,
    y=accuracy,
]{figures/plot_layers_1.csv};

\addplot+[
    discard if not={key}{Similarity_3},
    mark=o,
    color=blue,            %
]
table[
    col sep=comma,
    trim cells=true,
    x=size,
    y=accuracy,
]{figures/plot_layers_1.csv};

\addplot+[
    only marks,
    discard if not={key}{Unmodified},
    mark options={fill=red},
    mark=o,
    color=red
]
table[
    col sep=comma,
    trim cells=true,
    x=size,
    y=accuracy,
]{figures/plot_layers_1.csv};

\addplot+[
    only marks,
    discard if not={key}{Heuristic},
    mark=triangle*,
    mark options={fill=red,   %
                  draw=red, %
                  solid},
]
table[
    col sep=comma,
    trim cells=true,
    x=size,
    y=accuracy,
]{figures/plot_layers_1.csv};

\nextgroupplot[
    xlabel={},
    ylabel={Accuracy},
    ymin=0.3, ymax=0.6,
    title={$L=2$},
    xmin=0, xmax=36000,
    scaled x ticks=false,
    xticklabels={},
]

\addplot+[
    discard if not={key}{Greedy_0},
    mark=o,
    color=orange
]
table[
    col sep=comma,
    trim cells=true,
    x=size,
    y=accuracy,
]{figures/plot_layers_2.csv};
\addplot+[
    discard if not={key}{Greedy_1},
    mark=o,
    color=orange
]
table[
    col sep=comma,
    trim cells=true,
    x=size,
    y=accuracy,
]{figures/plot_layers_2.csv};
\addplot+[
    discard if not={key}{Greedy_2},
    mark=o,
    color=orange
]
table[
    col sep=comma,
    trim cells=true,
    x=size,
    y=accuracy,
]{figures/plot_layers_2.csv};

\addplot+[
    discard if not={key}{Greedy_3},
    mark=o,
    color=orange
]
table[
    col sep=comma,
    trim cells=true,
    x=size,
    y=accuracy,
]{figures/plot_layers_2.csv};

\addplot+[
    discard if not={key}{Similarity_0},
    mark=o,
    color=blue,
]
table[
    col sep=comma,
    trim cells=true,
    x=size,
    y=accuracy,
]{figures/plot_layers_2.csv};
\addplot+[
    discard if not={key}{Similarity_1},
    mark=o,
    color=blue,            %
]
table[
    col sep=comma,
    trim cells=true,
    x=size,
    y=accuracy,
]{figures/plot_layers_2.csv};
\addplot+[
    discard if not={key}{Similarity_2},
    mark=o,
    color=blue,            %
]
table[
    col sep=comma,
    trim cells=true,
    x=size,
    y=accuracy,
]{figures/plot_layers_2.csv};

\addplot+[
    discard if not={key}{Similarity_3},
    mark=o,
    color=blue,            %
]
table[
    col sep=comma,
    trim cells=true,
    x=size,
    y=accuracy,
]{figures/plot_layers_2.csv};

\addplot+[
    only marks,
    discard if not={key}{Unmodified},
    mark options={fill=red},
    mark=o,
    color=red
]
table[
    col sep=comma,
    trim cells=true,
    x=size,
    y=accuracy,
]{figures/plot_layers_2.csv};

\addplot+[
    only marks,
    discard if not={key}{Heuristic},
    mark=triangle*,
    mark options={fill=red,   %
                  draw=red, %
                  solid},
]
table[
    col sep=comma,
    trim cells=true,
    x=size,
    y=accuracy,
]{figures/plot_layers_2.csv};

\nextgroupplot[
    xlabel={Gates in Network},
    ylabel={Accuracy},
    ymin=0.3, ymax=0.6,
    title={$L=3$},
    xmin=0, xmax=36000,
    scaled x ticks=false,
    legend to name=sharedlegendtwo,
]

\addplot+[
    discard if not={key}{Greedy_0},
    mark=o,
    color=orange,
    forget plot
]
table[
    col sep=comma,
    trim cells=true,
    x=size,
    y=accuracy,
]{figures/plot_layers_3.csv};
\addplot+[
    discard if not={key}{Greedy_1},
    mark=o,
    color=orange,
     forget plot
]
table[
    col sep=comma,
    trim cells=true,
    x=size,
    y=accuracy,
]{figures/plot_layers_3.csv};
\addplot+[
    discard if not={key}{Greedy_2},
    mark=o,
    color=orange,
     forget plot
]
table[
    col sep=comma,
    trim cells=true,
    x=size,
    y=accuracy,
]{figures/plot_layers_3.csv};

\addplot+[
    discard if not={key}{Greedy_3},
    mark=o,
    color=orange
]
table[
    col sep=comma,
    trim cells=true,
    x=size,
    y=accuracy,
]{figures/plot_layers_3.csv};

\addlegendentry{Greedy Pruning (baseline)}

\addplot+[
    discard if not={key}{Similarity_0},
    mark=o,
    color=blue,
     forget plot
]
table[
    col sep=comma,
    trim cells=true,
    x=size,
    y=accuracy,
]{figures/plot_layers_3.csv};
\addplot+[
    discard if not={key}{Similarity_1},
    mark=o,
    color=blue,            %
     forget plot
]
table[
    col sep=comma,
    trim cells=true,
    x=size,
    y=accuracy,
]{figures/plot_layers_3.csv};
\addplot+[
    discard if not={key}{Similarity_2},
    mark=o,
    color=blue,            %
     forget plot
]
table[
    col sep=comma,
    trim cells=true,
    x=size,
    y=accuracy,
]{figures/plot_layers_3.csv};

\addplot+[
    discard if not={key}{Similarity_3},
    mark=o,
    color=blue,            %
]
table[
    col sep=comma,
    trim cells=true,
    x=size,
    y=accuracy,
]{figures/plot_layers_3.csv};

\addlegendentry{Similarity Pruning}

\addplot+[
    only marks,
    discard if not={key}{Unmodified},
    mark options={fill=red},
    mark=o,
    color=red
]
table[
    col sep=comma,
    trim cells=true,
    x=size,
    y=accuracy,
]{figures/plot_layers_3.csv};

\addlegendentry{Trivial Pruning}
\addplot+[
    only marks,
    discard if not={key}{Heuristic},
    mark=triangle*,
    mark options={fill=red,   %
                  draw=red, %
                  solid},
]
table[
    col sep=comma,
    trim cells=true,
    x=size,
    y=accuracy,
]{figures/plot_layers_3.csv};
\addlegendentry{Logic Equivalence}

\end{groupplot}

\end{tikzpicture}
\vspace{-0.4cm}

\begin{tikzpicture}
  \node[anchor=north, xshift=0cm] at (0,0) {%
    \pgfplotslegendfromname{sharedlegendtwo}%
  };
\end{tikzpicture}

    \caption{Test set accuracy of pruning techniques for differing layer interconnects being optimized.}
    \label{fig:pruning}
\end{figure}
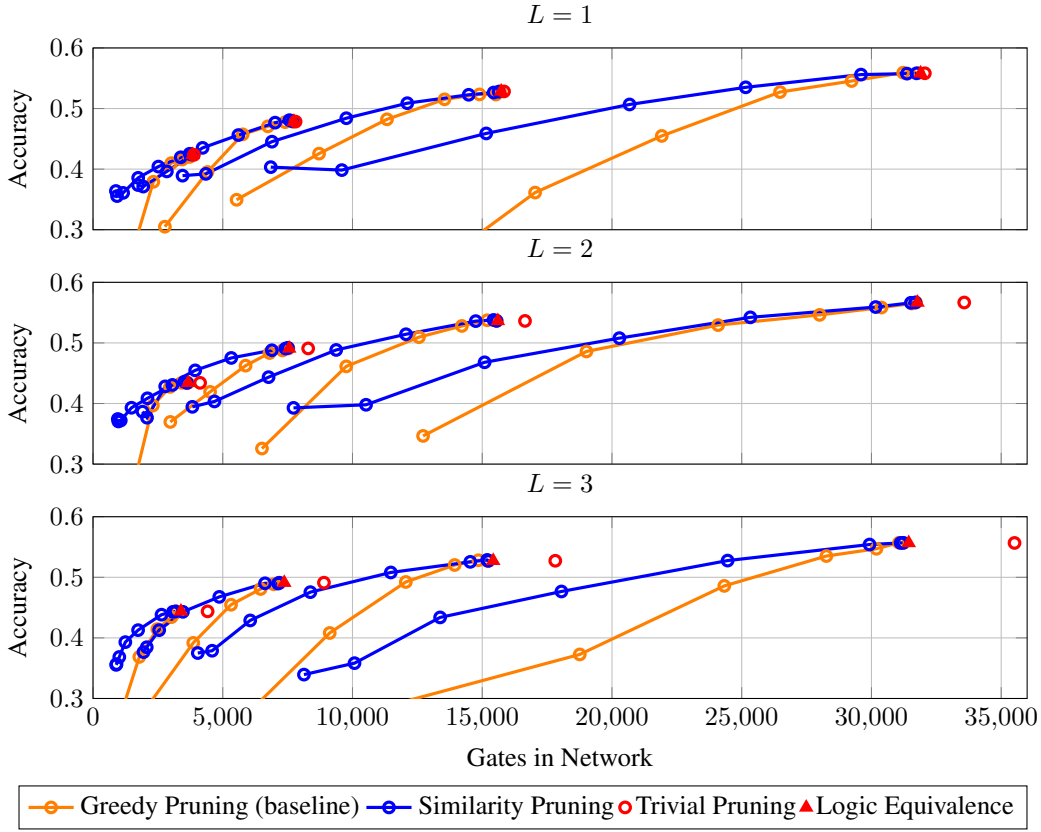

Figure~\ref{fig:pruning} shows test accuracy against the number of gates that survive post-training pruning for four layer widths \(1500\), \(3000\), \(6000\), \(12000\) and the three interconnect-training schedules \(L\in\{1,2,3\}\). Learned skip connections are not counted, but \(0\) and \(1\) gates are. Each curve starts at a red marker that represents the \emph{trivially} pruned network immediately after training—that is, with all disconnected gates removed but before any of the explicit pruning schemes are applied. The following paragraphs compare the three pruning strategies on these
checkpoints.

\paragraph{Logic equivalence pruning.}
When only the first interconnect matrix is learned (\(L=1\)), virtually no additional gates are removed by the SAT-based logic equivalence pass, suggesting that the training procedure already avoids redundant symbolic structures. Extending interconnect optimisation beyond the first layer (\(L=2,3\)) introduces more duplicate functions, and the SAT check therefore eliminates a larger absolute number of gates. However, the \emph{final} gate count converges to roughly the same number of gates as for \(L=1\), implying that logic equivalence pruning merely neutralises the duplicates created by deeper interconnect learning rather than yielding a net reduction below the \(L=1\) baseline. Table \ref{tab:logic_eq} shows the number of gates after trivial and logic equivalence pruning for the largest network.

\begin{table}[t] 
\centering
\caption{Gate counts \emph{after} trivial pruning and \emph{after} the subsequent logic equivalence pass. One seed is shown.}
\label{tab:logic_eq}
\begin{tabular}{cccc}
\toprule
$L$ &
test acc. & trivial &
logic equiv. \\ \midrule
1 &0.558 &32\,058 & 31\,897 \\
2 &0.566 &33\,575 & 31\,774 \\
3 &0.556 &35\,514 & 31\,436 \\ \bottomrule
\end{tabular}
\end{table}

\paragraph{Greedy \& Similarity pruning.}
Of the two data-driven schemes, similarity pruning consistently preserves accuracy better than the greedy baseline.  On our smallest configuration—three layers totaling \(4\,500\) gates (\(1\,500\) per layer)—similarity pruning can compress a \(9\,000\)-gate model (\(3\,000\) per layer) down to the same approximate number of gates while matching the accuracy of a \(4\,500\)-gate network trained from scratch.  The same qualitative pattern holds at larger widths: accuracy degrades much more gracefully under similarity pruning than under greedy pruning. %

\section{Related Work}

We review prior work in three areas relevant to this paper: differentiable training of Boolean networks, learnable interconnect mechanisms, and pruning/compression techniques.

\paragraph{Differentiable Boolean Networks}  
Neural Architecture Search (NAS) has popularized the idea of relaxing discrete architectural choices into a continuous domain to enable gradient‐based optimization. Differentiable Architecture Search (DARTS) \cite{liudarts} uses Softmax to mix cell‐ or layer‐level operations, jointly learning topology and weights. At a finer granularity, Chen et al. \cite{chen2020learning} employ Gumbel‐Softmax sampling to learn symbolic expression graphs by selecting among candidate operators in a differentiable manner. These approaches demonstrate that discrete connectivity can be learned end‐to‐end—a principle that has subsequently been adapted at the level of individual Boolean gates.

Building on this principle, Petersen et al. \cite{petersen2022deep} were the first to apply it directly to Boolean circuits. In DLNs, each 2-input gate is represented during training as a Softmax‐normalized mixture over its 16 truth tables, enabling end-to-end gradient descent; at inference time each gate is “hardened” to its most probable function, yielding fast bitwise circuits.  Kim et al. \cite{kim2023stochastic} later replaced the Softmax mixture with a Gumbel-Softmax distribution \cite{maddison2017concrete, jang2017categorical}, improving gate‐selection learning. DLNs have shown promising performance in convolutional \cite{petersen2024convolutional} and recurrent \cite{Pietro2025LogicCA} settings. 

More recently, Bacellar et al. \cite{bacellar2024differentiable} proposed DWNs, which map Boolean functions to FPGA‐friendly Look-up tables (LUTs) and train them via an Extended Finite Difference method, yielding competitive performance and improved deployment possibilities. Trained DWNs with 2-input LUT tables can be converted to DLNs and vice versa, making both pruning techniques and interconnect learning methods applicable to each. Orthogonally, Benamira et al. \cite{DBLP:conf/ijcai/BenamiraPYGH24} convert trained real‐valued networks into Boolean layers—though the resulting circuits are much larger.

\paragraph{Learnable Interconnect in Boolean Networks}  
Original DLNs \cite{petersen2022deep,petersen2024convolutional} use a fixed, random connection pattern, leaving room for performance improvements. Both Yue et al. \cite{yue2024learninginterpretabledifferentiablelogic} and the DWNs introduced by Bacellar et al.~\cite{bacellar2024differentiable} employ a straight‐through estimator (STE) to select the argmax connections in the forward pass while backpropagating through the Softmax. Yet the interconnect in \cite{bacellar2024differentiable} scales with both input layer size and the number of Boolean functions in each layer, making it impractical for the large Boolean layers required in image tasks. Although \cite{yue2024learninginterpretabledifferentiablelogic} introduces a limited subset of possible connections, they do not consider resampling this subset.

\paragraph{Pruning and Compression}
Magnitude‐based pruning \cite{han2015learning} and other pruning methods (for a survey, see \cite{cheng2024survey}) have long been used to compress overparameterized floating-point neural networks by removing low‐importance weights. More recent mask‐based approaches \cite{kang2020operation} optimize continuous masks over connections, akin to our interconnect learning mechanism. In the DBN context, pruning can eliminate redundant gates or connections whose outputs are constant or unused, yielding more compact circuits. Saliency-map-based interpretability and pruning with DLNs were explored by Wormald et al. \cite{wormald2025explogic}, albeit at some cost to accuracy. Additionally, it is unclear if their methods scale beyond the MNIST dataset.

\section{Summary \& Discussion}

Optimising the \emph{first‐layer} interconnect consistently raises test-set accuracy, and the proposed parameter-efficient learning scheme removes the memory and runtime bottlenecks that previously limited the utilization of such an approach. For deeper layers the picture changes.  As we optimise the interconnect of deeper layers, the model begins to overfit, suggesting that the expressive power unlocked by flexible wiring must be balanced by stronger regularisation. Previous work on DBNs explored some techniques to reduce overfitting, including weight decay~\cite{petersen2024convolutional}, spectral regularization~\cite{bacellar2024differentiable}, and data augmentation~\cite{bacellar2024differentiable}; a systematic evaluation of these remains open.

Pruning plays two distinct roles. Logic equivalence pruning is lightweight insurance: it eliminates the redundancy introduced by multi-layer interconnect learning. Data-driven pruning is more aggressive and may hurt accuracy. However, in the case of smaller networks, utilizing our similarity-based pruning strategy can recover the accuracy of a from scratch trained network.

Finally, our study only considers the CIFAR-10 dataset. Extending the analysis to larger image sets or to non-image tasks will be essential for clarifying how broadly these observations generalise. We do not perform a comparison with the pruning scheme introduced by~\cite{wormald2025explogic}, which was only applied to MNIST networks with a smaller input space. None of our pruning schemes exploit label information; doing so might reveal additional redundancies.

\section{Contributions}

In summary, we: (i) introduce a scalable differentiable interconnect for wide DBNs, (ii) provide the first systematic study of multi-layer interconnect learning, (iii) formalize provably sound logic equivalence pruning for DBNs, and (iv) identify the proposed similarity pruning as outperforming our magnitude-based baseline.

Our results also highlight susceptibility to overfitting in DBNs, underscoring the need for research in this direction

\bibliographystyle{plain}
\bibliography{references}

\newpage
\appendix

\section{Appendix}
\subsection{Hyperparemeter Sensititvity}

For hyperparameter evaluation, we vary only along one dimension, otherwise, we use the default values: $C = 8$, $R = 4$, and $\beta = 20$. We do not ablate $R$, instead we set it to $R=\frac{C}{2}$. All runs are performed on an NVIDIA A10 GPU. The models are investigated on CIFAR-10 with a layer-width of 12000. Figures \ref{fig:beta} and \ref{fig:connectivity} show that our method is largely unaffected by differing hyperparameter selection. All experiments are performed over 3 seeds.

\begin{figure}[h!]

  \centering
  \includegraphics[width=0.5\linewidth]{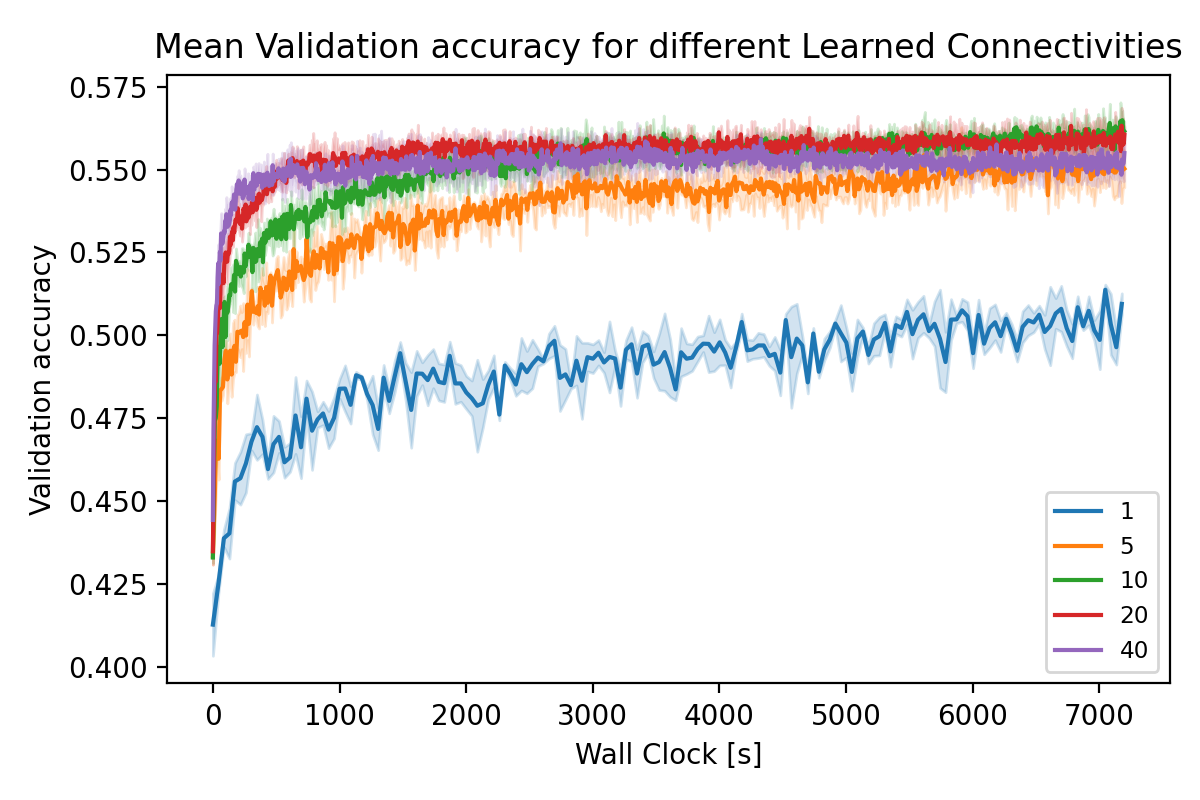}%
    \caption{Varying $\beta$ values ablated. It appears that larger beta values (i.e. more seldom updating of gate parameters) are beneficial at the start of training and lead to faster gains. However, later in training, faster updates are beneficial. Future work should consider a scheduled $\beta$ rate.}
    \label{fig:beta}
\end{figure}%

\begin{figure}[h!]

  \centering
  \includegraphics[width=0.5\linewidth]{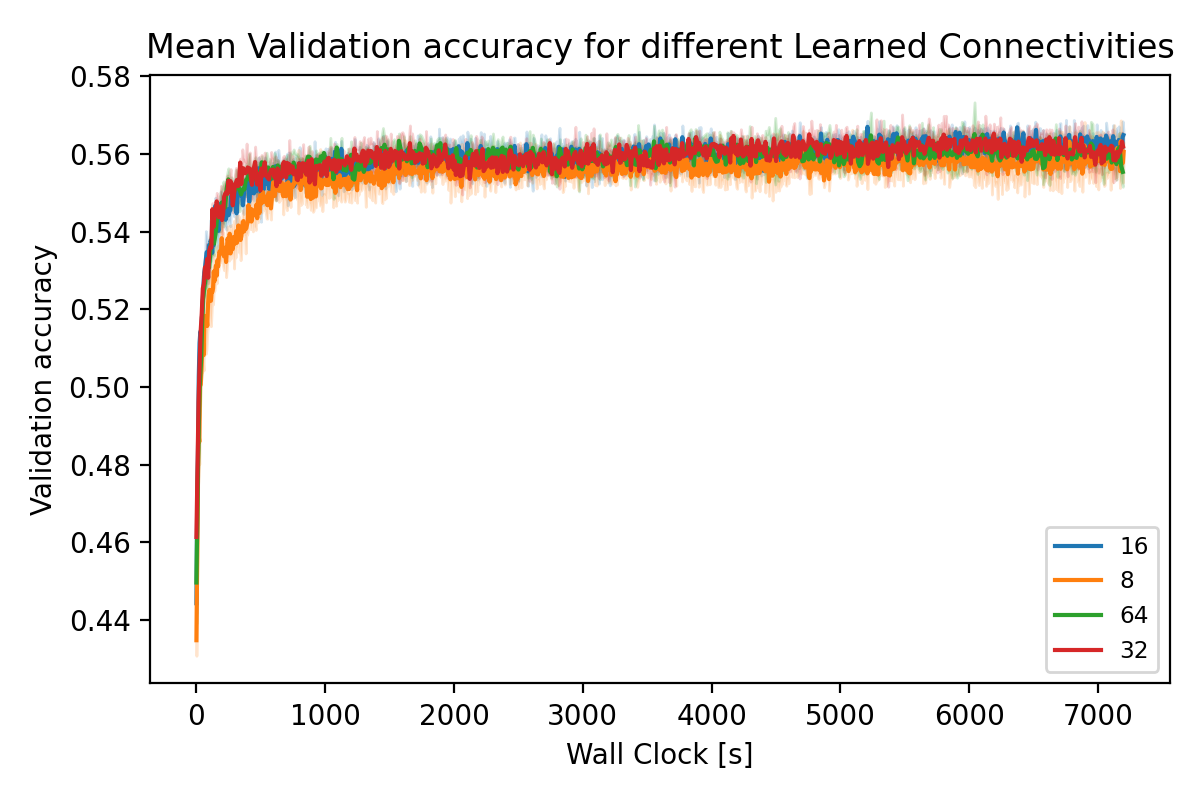}%
    \caption{Varying $C$. A higher connectivity leads to more memory requirements, but leads to faster initial convergence. Overall, similar validation accuracies are attained.}
    \label{fig:connectivity}
\end{figure}%

\subsection{Convergence Time in Epochs}

Figure \ref{fig:convergence in epochs} shows the convergence time in epochs for our investigated interconnect learning schemes.
\begin{figure}[h!]
    \centering
    \includegraphics[width=0.8\linewidth]{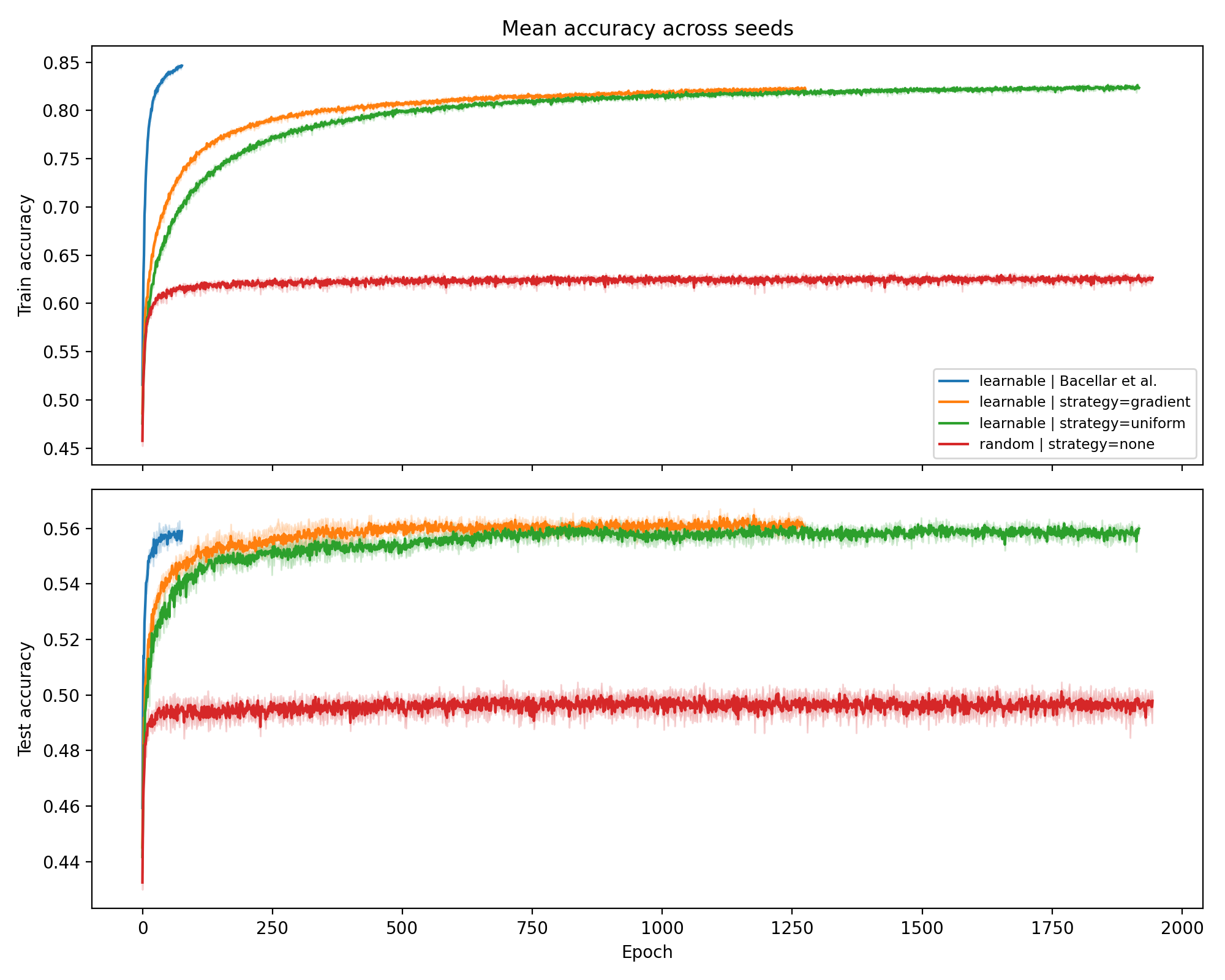}
    \caption{Accuracy vs. number of epochs for our investigated interconnect schemes. All schemes were run for two hours. }
    \label{fig:convergence in epochs}
\end{figure}

\end{document}